\documentclass{article}
\usepackage{psfrag}
\usepackage{tikz,pgf}
\usetikzlibrary{arrows,decorations,backgrounds}
\tikzstyle{place}=[circle,draw=black,draw=blue!50,fill=blue!10,inner sep=0mm, minimum size=6mm]
\tikzstyle{transition}=[rectangle,draw=black!50,fill=black!20,thick]

\newcommand{\Ckt} {\CSet_k^t}
\newcommand \Bkp {B_{k-1}}

\usepackage{amsthm,amsfonts,amsbsy,amsmath,amssymb,mathrsfs,dsfont}
\usepackage[round]{natbib}
\usepackage{graphicx}
\usepackage{subfigure}
\usepackage{paralist}
\usepackage{algorithm}
\usepackage{algorithmicx}
\usepackage{algpseudocode}
\theoremstyle{plain} 

\newcommand\Reals {{\mathds{R}}}

\newcommand\SC {{\mathscr{C}}}

\newcommand\FB {{\mathfrak{B}}}
\newcommand\FD {{\mathfrak{D}}}

\newcommand\CP {{\mathcal{P}}}
\newcommand\CC {{\mathcal{C}}}
\newcommand\CF {{\mathcal{F}}}

\newcommand\CX {{\mathcal{X}}}
\newcommand\CY {{\mathcal{Y}}}

\newcommand\CV {{\mathcal{V}}}

\newcommand\CW {{\mathcal{W}}}

\newcommand\bx {{\mathbf{x}}}
\newcommand\cvr {\CC}
\newcommand\ctx {M}
\newcommand\ctxB {N}
\newcommand\ctxC {K}

\newcommand \Stop {S}
\newcommand \st {\Stop_t}
\newcommand \sk {\Stop_k}

\newcommand \skp {\Stop_{k-1}}
\newcommand \skn {\Stop_{k+1}}

\newcommand \Bkt {B_k^t}
\newcommand \Bkpt {B_{k-1}^t}

\newcommand \len[1] {\ensuremath{\ell\left(#1\right)}}

\newcommand \loss {L_t}

\newcommand \BigO[1] {O\left(#1\right)}
\newcommand \BigW[1] {\Omega\left(#1\right)}

\newcommand \Borel[1] {\FB_{#1}}
\newcommand \Dist[1] {\FD_{#1}}

\newcommand \model {\mu}

\newcommand \Weights {\CW}
\newcommand \Trans {\CV}
\newcommand \tran {v}
\newcommand \tranS {v_\ctx}
\newcommand \weight {w}

\newcommand \wS {\weight_\ctx}

\newcommand \hwS {\widehat{\weight_\ctx}}
\newcommand \hwQ {\widehat{\weight_\ctxB}}
\newcommand \hvS {\widehat{\tran_\ctx}}

\newcommand \wkt {\weight_{\ctx}^t}
\newcommand \wktn {\weight_\ctx^{t+1}}
\newcommand \proj {f}

\newcommand\set[1] {\left\{#1\right\}}
\newcommand\cset[2] {\left\{#1 \mathrel{:} #2\right\}}
\newcommand\cseq[2] {\left(#1 \mathrel{:} #2\right)}

\renewcommand \Pr {\mathop{\mbox{\ensuremath{\mathbb{P}}}}\nolimits}

\newcommand \defn {\mathrel{\triangleq}}

\newcommand \bel {\xi}
\newcommand \mbel {\psi}

\newcommand \post {p}
\newcommand \Posts {\CP}

\newcommand \prior {\phi}
\newcommand \Priors {\Phi}

\newcommand \bra {\beta}

\newcommand \seqt[2] {#1^{#2}}

\newcommand \xtt {\seqt{x}{t}}
\newcommand \ytt {\seqt{y}{t}}

\newcommand \xttn {\seqt{x}{t+1}}
\newcommand \yttn {\seqt{y}{t+1}}

\newcommand \xtn {x_{t+1}}
\newcommand \ytn {y_{t+1}}

\DeclareMathAlphabet{\mathpzc}{OT1}{pzc}{m}{it}

\newcommand \Bernoulli {{\mathpzc{Bern}}}

\newcommand \Multinomial {{\mathpzc{Mult}}}

\newcommand \depth {\ensuremath{d}}

\newcommand \Contexts {\SC}
\newcommand \CSet {{\mathcal{C}}} 
\newcommand \CSeq {\mathfrak{C}}

\newcommand {\suffix} {\prec}

\newcommand {\emptyseq} {\mathbf{0}}

\newtheorem {definition}{Definition}
\newtheorem {theorem}{Theorem}
\newtheorem {lemma}{Lemma}
\newtheorem {example}{Example}

\def\clap#1{\hbox to 0pt{\hss#1\hss}}

\def\mathclap{\mathpalette\mathclapinternal}

\def\mathclapinternal#1#2{%
           \clap{$\mathsurround=0pt#1{#2}$}}

\usepackage{natbib}

\title{Context models on sequences of covers}
\author{Christos Dimitrakakis}
\begin{document}
\maketitle
\begin{abstract}
  We present a class of models that, via a simple construction,
  enables exact, incremental, non-parametric, polynomial-time,
  Bayesian inference of conditional measures. The approach relies upon
  creating a sequence of covers on the conditioning variable and
  maintaining a different model for each set within a cover. Inference
  remains tractable by specifying the probabilistic model in terms of
  a random walk within the sequence of covers. We demonstrate the
  approach on problems of conditional density estimation, which, to
  our knowledge is the first closed-form, non-parametric Bayesian
  approach to this problem. 
\end{abstract}

\section{Introduction}
\label{sec:introduction}

Conditional measure estimation is a fundamental problem in statistics.
Specific instances of this problem include classification, regression
and conditional density estimation.  This paper formulates a general
approach for non-parametric, incremental, closed-form Bayesian
estimation of conditional measures that relies on a model structure
defined on a sequence of covers.  This is an important development,
particularly for the problem of conditional density estimation, where
although non-parameteric kernel-based approaches that currently
dominate generally perform well, a fast, tractable, incremental,
Bayesian approach has been lacking. 

This construction used in this paper employs a random walk in a set of
contexts. In its simplest form, this can be seen as a descendant of
context tree methods for variable order Markov
models~\citep{willems:context,dimitrakakis:aistats:2010} and Bayesian
non-parametric methods for tree-based density estimation
approaches~\citep{hutter:bayestree,wong2010optional}. These approaches
utilise a stopping variable construction on a tree to simplify
inference.  The central contribution of this paper is to generalise
this to a terminating random walk on a lattice. Then the inference
procedure remains tractable, while the lattice structure increases the
flexibility and applicability of the model.  As an example, the
proposed framework is applied to the important problem of conditional
density estimation, obtaining the first closed-form, incremental,
non-parametric Bayesian approach to this problem. 

Stated generally, the problem of incremental, conditional measure
estimation in a Bayesian setting is as follows. We observe the
sequences $x^t = \cseq{x_i}{i=1,\ldots,t}$ and $y^t =
\cseq{y_i}{i=1,\ldots,t}$, with $x_i \in \CX$ and $y_i \in \CY$.
Informally, our goal is the prediction of the next observation $\ytn$
given the next conditioning variable $\xtn$ and all previous evidence
$\xtt, \ytt$.  More precisely, we wish to calculate the probability
measure:
\begin{align}
  \psi_t(Y \mid \xtn ) \defn \Pr(\ytn \in Y \mid \xttn, \ytt)
\end{align}
for all $Y \in \Borel{\CY}$, where $\Borel{\CY}$ denotes
the Borel sets of $\CY$, through Bayes' theorem.

The main idea we use to tackle this problem is to first define a
sequence of covers on the space of all sequences $x^t$. Each cover is
a collection of sets, such that for any sequence $x^t$ there exists at
least one set $c$ in every cover containing that sequence. In
addition, each set $\ctx$ corresponds to a model $\prior_\ctx$ on
$\CY$.  In order to combine these, we introduce a random variable $\st
\in \Contexts$, such that $\bel_t(\ctx \mid \xtn) \defn \Pr(\st = \ctx
\mid \xttn, \ytt)$, is the probability of the model $\prior_\ctx$.
Then the conditional measure:
\begin{equation}
\mbel_t(Y \mid \xtn) = \sum_{\ctx \in \Contexts} \prior_\ctx^t(Y \mid
\xtn) \bel_t(\ctx \mid \xtn),
\label{eq:marginal}
\end{equation}
can be readily obtained via marginalisation over the set of contexts.

We show that via the sequence of covers, $\xi$ can be specified in
terms of a random walk. This allows closed-form, incremental inference
to be performed in polynomial time for conditional densitiy estimation
and variable order Markov models, by selecting the covers
appropriately. The resulting class of models allows the introduction
of several other interesting model classes.

\section{Context models}
\label{sec:context-models}
We first introduce some notation and basic assumptions.  Unless
otherwise stated, we assume that all sets $\CX$ are measurable with
respect to some $\sigma$-algebra $\Borel{\CX}$. We denote sequences of
observations $x_i \in \CX$ by $\xtt \defn
\cseq{x_i}{i=1,\ldots,t}$. The set $\CX^0 \defn \set{\emptyseq}$
contains only the null sequence $\emptyseq$, while $\CX^n \defn
\times^n \CX$ denotes the sequences of length $n$ and
the set of all sequences is denoted by $\CX^* \defn \bigcup_{n=0}^\infty \CX^n$.
Finally, we denote the length of any sequence $x \in \CX^*$ by
$\len{x}$ such that $x \in \CX^{\len{x}}$.

A {\em cover} $\cvr$ of some set $A$ is a collection of sets such that
$\bigcup_{\ctx \in C} \ctx \supset A$. A {\em refinement} $\cvr'$ of
$\cvr$ is a cover of $A$ such that for any $\ctx' \in \cvr'$, there is
some $\ctx \in \cvr$ such that $\ctx' \subset \ctx$. We consider
models constructed on a {\em sequence} of covers $\CSeq \defn
\cseq{\CSet_k}{k=1,\ldots}$ of $\CX^*$.
Letting $\Contexts \defn \bigcup_k \CSet_k$ be the collection of all
subsets in our sequence of covers, we refer to each subset $\ctx \in
\CSet$ as a {\em context}. Partition trees, where each cover is
disjoint and a refinement of the previous cover, are an interesting
special case:
\begin{example}[Binary alphabet] 
  \label{ex:binary}
  Let $\CX = \set{0,1}$. For $k =
  1, 2, \ldots$, let $\CSet_k$ be the partition of $\CX^*$ into
  $2^{k-1}$ subsets, with the following property. For all $\ctx \in
  \CSet_k$, and any $a, b \in \CX^*$: $a, b \in \ctx$ if and only if
  $a_{\len{a} - i} = b_{\len{b} - i}$ for all $i = 0, \ldots, k -
  1$. This creates a sequence of partitions based on a suffix tree and
  can be used in the development of variable order Markov models.
\end{example}
\begin{example}[Unit interval]
  \label{ex:unit}
  Let $\CX = [0,1]$. For $k = 1, 2,
  \ldots$, let $\CSet_k$ be the partition of $\CX^*$ into $2^{k-1}$
  subsets, $\ctx_{k,i} \defn \cset{x \in \CX^*}{x_{\len{x}} \in
    [2^{k-1}(i-1), 2^{k-1}i)}$. A generalised form of this sequence of
  covers is used in the construction of conditional density estimation
  using the proposed construction, and shall be the main focus of the
  current paper.
\end{example}
We now describe a conditional measure on $\CY$ indexed by $\CX^*$
defined on such a structure.  This will form the basis for conditional
measure estimation.  Intuitively, the structure defines a set of
probability measures on $\CY$, indexed by the set of all contexts. The
structure is such that, for any $x \in \CX$ there is only one
corresponding context $f(x)$, even if there are many contexts
containing $x$. The contexts themselves have the property that the
corresponding context for any $x$ in the set they define is either the
same context or one of the subsequent contexts in the sequence of
covers. This will be useful later, since it will allow us to perform
closed form inference on a distribution of context models.
\begin{definition}
  A context model $\model = (\Posts, \proj)$ defined on a (countable)
  sequence of covers $\CSeq = \cseq{\CSet_k}{k=1,\ldots,}$ of $\CX^*$,
  is composed of:
  \begin{enumerate}
  \item A set $\Posts$ of ``local'' probability measures on $\CY$,
    conditional on $\CX^*$ and indexed by elements in the set of
    contexts $\Contexts = \bigcup_k \CSet_k$:
    \begin{align}
      \Posts &\defn \cset{\post_\ctx(\cdot \mid x)}{\ctx \in \Contexts}, & x
      &\in \CX^*.
      \label{eq:posterior-model}
    \end{align}
  \item A context map $\proj : \CX^* \to \CSet$ such that $\forall x
    \in \CX^*$, if $f(x) \in \CSet_k$, then for any $x' \in f(x)$ it
    holds that $f(x') \cap f(x) \neq \emptyset$ and $f(x') \in
    \CSet_{k + h}$ with $h \geq 0$.
  \end{enumerate}
  The model $\model$ specifies the following conditional measure on
  $\CY$ for any $x \in \CX^*$:
    \begin{align}
      \label{eq:conditional-measure}
      \Pr_\model(Y \mid x) 
      &= 
      \post_{f(x)}(Y \mid x),
      &
      Y &\subset \CY.
    \end{align}
\end{definition}
Though the local measures $\Posts$ can be simple, so that inference
can be efficient, the model's overall complexity will depend on the
context map and cover structure. 

We now describe a distribution of such models, whereby exact Bayesian
inference can be performed in polynomial time.  Intuitively, the
distribution can be seen as a two-stage process. Firstly, we sample a
context map $\proj$ from a set of context maps $\CF$, through a
halting random walk on the set of all contexts.  Secondly, for each
context $\ctx$ we sample a conditional measure $\post_\ctx$ from a
distribution $\prior_\ctx$.  The construction of and sampling from
this distribution, are discussed in Sec.\ref{sec:construction}, while
Sec.~\ref{sec:marginal} shows how to sample from marginal distribution
$\mbel$ and Sec.~\ref{sec:inference} derives the inference procedure.

\subsection{Construction of the context model distribution}
\label{sec:construction}

\begin{definition}[Cover model]
  A cover model defines a distributilateon $\bel$ on context models
  $\model = (\Posts, \proj)$, through a tuple $(\CSeq, \Weights,
  \Trans, \Priors)$, where:
  \begin{enumerate}
  \item $\CSeq \defn \cseq{\CSet_k}{k=1,\ldots}$ is a sequence of
    covers, and $\Contexts \defn \cset{\ctx \in \CSet}{\CSet \in \CSeq}$
    is the set of all contexts in each cover.
  \item $\Weights = \cset{\wS}{\ctx \in \Contexts}$, with $\wS \in
    [0,1]$, is a set of stopping probabilities.
  \item $\Trans = \cset{\tranS}{\ctx \in \Contexts}$ is a set of
    transition probability vectors, such that: $\|\tranS\|_1 = 1$, and
    that if $\ctx \in \CSet_k$, then $\tran_{\ctx,\ctxB} \in [0,1]$
    for all $\ctxB \in \CSet_{k-1}$ such that $\ctxB \cap \ctx \neq
    \emptyset$ while $\tran_{\ctx, \ctxB} = 0$ otherwise
  \item $\Priors = \cset{\prior_\ctx}{\ctx \in \Contexts}$, is a set
    of priors such that each $\prior_\ctx$ is a probability measure on
    $\Dist{\CY \mid \CX}$, where $\Dist{\CY \mid \CX} \defn
    \cset{p_\theta(\cdot \mid x)}{\theta \in \Theta}$ is a set of
    probability measures on $\CY$, conditional on $x \in \CX$ and
    parameterised in $\Theta$.
  \end{enumerate}
\end{definition}
In order to sample a context model $\model = (\Posts, \proj)$ from
$\bel = (\CSeq,\Weights, \Trans, \Priors)$, we draw $\Posts$ directly from
$\Priors$, while we construct $\proj$ via two auxiliary variables $\hwS, \hvS$
drawn respectively from a Bernoulli and a multinomial distribution:
\begin{subequations}
  \begin{align}
    \post_\ctx &\sim \prior_\ctx \\
    \hwS & \sim \Bernoulli(\wS) \\
    \widehat{\tran_{\ctx}} & \sim \Multinomial(\tran_{\ctx}).
  \end{align}
\end{subequations}
These draws are performed independently for all $\ctx \in \Contexts$.
The construction of $\proj$ relies on the cover structure.  For any
$\xtt \in \CX^*$, we denote the collection of contexts at depth $k$
containing $\xtt$ by
\begin{align}
  \CSet_k^t &\defn \cset{\ctx \in \CSet_k}{\xtt \in \ctx}.
  \label{eq:matching-contexts}
\end{align}
We then define the context map $\proj$ as follows: $\proj(x^t) = \ctx
\in \CSet_k^t$, if and only if $\hwS = 1$ and $\hwQ = 0$ for all
$\ctxB \in \CSet_{h}^t$ with $h < k$.

\subsection{Drawing samples from the marginal distribution}
\label{sec:marginal}
In order to generate an observation in $\CY$ from the marginal
distribution derived from $\xi$, we can perform the following random
walk.
\begin{definition}[Marginal samples]
  We perform a random walk on the sequence of covers $\CSeq =
  \cseq{\CSet_k}{k=1,\ldots,\depth}$, with parameters $\Weights, \Trans$,
  generates a random sequence $\Stop_1, \ldots, \Stop_K$, with $K \in
  \set{1,\ldots, \depth}$, such that at each stage $k$,
  \begin{enumerate}
  \item $\sk \in \CSet_{\depth + 1 - k}$ for all $k$. 
  \item With probability $\weight_{\sk}$, the walk stops and we
    generate a local model $\prior$ from $\psi_{\sk}$ and subsequently
    an observation $x$ from $\prior$.
  \item Otherwise, $\skn = \ctxB$ with probability $\tran_{\sk,\ctxB}$, for
    all $\ctxB \in \CSet_{d+k}$.
  \end{enumerate}
\end{definition}

\subsection{Inference}
\label{sec:inference}
At time $t$, we have observed $\xtt = \cseq{x_i}{i=1,\ldots, t}$ and
$\ytt = \cseq{y_i}{i=1,\ldots, t}$ our model now has parameters
$\Weights_t, \Trans_t$, describing a distribution over context
models. We wish to update these parameters in the light of new
evidence $\xtn, \ytn$. The main idea is to use a random walk that
halts at some context $M_t$, in order to marginalise over context
models.  By definition, for any observation sequence, there is at
least one context containing $\xtt$ in every cover $\CSet_k$.  We
denote the collection of those contexts by $\CSet_k^t$, as in
\eqref{eq:matching-contexts}.

We start each stage $k$ of the walk at a context $\sk = \ctxB \in
\Ckt$ and proceed to $k -1, k-2, \ldots, 1$. Let $\Bkt \defn
\set{\ctx_t \in \bigcup_{j=1}^k C^t_j}$ denote the event that the walk
stops in one of the first $k$ stages.  Then, with probability
$w_\ctx^t$, we generate the next observation from the context $\ctxB
\in \CSet_k^t$, so that $\ytn \mid x^{t+1} \sim \prior_\ctxB^t(\cdot
\mid x_{t+1})$. Otherwise, we proceed to the next stage, $k+1$, by
moving to context $\ctxC \in C_{k+1}^t$ with probability
$v^t_{\ctxB,\ctxC}$.  More precisely:
\begin{align}
  \label{eq:sparse-weight}
  v^t_{\ctxB,\ctxC} &\defn \Pr(\skp = \ctxC \mid \sk = \ctxB, \xtt), 
  \\
  \label{eq:stop-parse}
  w_\ctxB^t 
  &=
  \Pr(\ctx_t \in C_k \mid \sk = \ctxB, B^t_k, \xtt).
\end{align}

The central quantity for tractable inference in this model is the
marginal prediction given the event $\Bkt$, for which we can obtain
the following recursion:
\begin{multline}
  \label{eq:tree-recursion}
  \psi^t_\ctxB (\ytn | \xtn)  \defn \Pr(\ytn | \sk {=} \ctxB, B^t_k, \xttn) 
  \\
  =
  w_\ctxB^t \prior_\ctxB^t(\ytn | \xtn)
  +
  (1 - w_\ctxB^t) \Pr(\ytn | \sk {=} \ctxB,  B^t_{k-1}, \xttn),
\end{multline}
noting that if we do not stop at level $k$ then $B^t_{k-1}$ is
trivially true, or more precisely, if $\Bkt$ and $\ctx_t \notin
\CSet_k^t$ then $\Bkpt$.  Furthermore, it is easy to see that:
\begin{align}
  \Pr(\ytn \mid \sk=\ctxB,  B^t_{k-1}, \xttn)
  &=
  \sum_{\mathclap{\ctxC \in \CSet_{k-1}^t}} \psi_\ctxC^t(\ytn \mid \xtn) v_{\ctxB,\ctxC}^t.
\end{align}
We can now calculate the stopping probabilities $w$ and the transition
probabilities $v$ given the new evidence
as follows:
\begin{theorem}
  Given a set of stopping parameters $\Weights_t =
  \cset{w_\ctx^t}{\ctx \in \Contexts}$, a set of transition
  parameters $\Trans_t = \cset{v_{\ctxC,\ctxB}^t}{\ctxC \in \CSet_k,
    \ctxB \in \CSet_{k-1}, k = 1, \ldots}$ and a set of local measures on $\CX$:
  $\cset{\prior_\ctxC^t}{\ctxC \in \Contexts}$, then the parameters at
  the next time step are given by:
  \begin{equation}
    v^{t+1}_{\ctxB, \ctxC}
    =
    \frac{\psi_\ctxB^t(\ytn \mid \xtn)
      v^t_{\ctxB, \ctxC}}
    {\sum_\ctx \psi_\ctx^t(\xtn)
      v^t_{\ctx.\ctxC}}
  \end{equation}
  and
  \begin{equation}
    \wktn
    = \frac{\psi_k^t(\ytn \mid \xtn) \wkt}
    {\psi_k^t(\ytn \mid \xtn) \wkt
      + \Pr(\ytn|\xttn, \sk {=} \ctx, \Bkpt) (1 - \wkt)},
  \end{equation}
  where $\psi$ is given by \eqref{eq:tree-recursion}, while
  $\prior_\ctx^t$ is a marginal measure conditioned on the first $t$
  observations for which $\ctx$ is reachable by the random walk.
  \label{the:inference}
\end{theorem}
\begin{proof}
  The proof mainly follows straightforwardly from the previous
  development. From Bayes theorem and \eqref{eq:sparse-weight}, we
  obtain the recursion:
  \begin{align*}
    v^{t+1}_{\ctx,\ctxB} &\defn
    \Pr(\skn {=} \ctxB \mid  \sk {=} \ctx, \yttn, \xttn)
    \\
    &=
    \frac{\Pr(\ytn \mid \skn {=}\ctxB, \sk {=} \ctx, \ytt, \xttn) 
      \tran_{\ctx,\ctxB}^t
    }
    {\sum_{\ctxC \in \Contexts}
      \Pr(\ytn \mid \skn {=}\ctx, \sk {=} \ctx, \Bkp, \ytt, \xttn) 
      \tran_{\ctx,\ctxC}^t
      }
  \end{align*}
  Since the random walk $\sk$ is first order~\footnote{We note that a
    higher order random walk on $\sk$ is possible, but we do not consider it in this paper.}
  \[
  v^{t+1}_{\ctx,\ctxB} =
  \frac{\Pr(x_{t+1} \mid \skn =\ctxB, \Bkp, \xtt) 
    v^t_{\ctx,\ctxB}}
  {\sum_c \Pr(x_{t+1} \mid \skn =\ctx, \Bkp, \xtt) 
    v^t_{c,\ctxB}},
  \]
  while finally from \eqref{eq:tree-recursion} we obtain the required
  result.  The recursion for $\wS^{t+1}$ is proven analogously to
  Theorem~1 in \citep{dimitrakakis:aistats:2010}.
\end{proof}

\subsection{Complexity}
\label{sec:complexity}

As previously mentioned, the overall complexity of the model depends
on how the sequence of covers is constructed. The more dense the
covers are, the higher the computational complexity.  In the worst
case scenario, all contexts are reachable by the a random walk,
bringing complexity to linear in the number of total contexts. More
generally, we can relate the complexity to the growth $\zeta$ of the
number of sets containing each sequence $x \in \CX^*$ as the number of
covers $\depth$ increases.
\begin{lemma}
  Let the sequence of covers be of length $\depth$. For any $x \in \CX^*$,
  let $\CSet_k(x) = \cset{\ctx \in \CSet_k}{x \in \ctx}$ be the set of
  contexts containing $x$ in the cover $\CSet_k$ and let
  $|\CSet_k(x)|$ be the number of contexts in $\CSet_k(x)$. If there
  exists $\zeta > 0$ such that, for any $x \in \CX^*$
  \[
  |\CSet_{k+1}(x)| \leq \zeta |\CSet_{k}(x)|,
  \]
  then the number number of reachable contexts is bounded by
  $\BigO{\frac{\zeta^{\depth+1}-1}{\zeta-1}}$.
\end{lemma}
\begin{proof}
  The proof follows trivially by the geometric sequence.
\end{proof}

\section{Applications}
\label{sec:applications}
The class contains both variable order Markov models and mixtures of
$k$-order Markov models on discrete alphabets, as well as density
estimators and conditional density estimators. All that is required in
order to apply the method to various cases is to select the context
structure and the priors on the random walk, stopping probabilities,
appropriately. 

\subsection{Variable order Markov models}
\label{sec:vmm}
In the variable order Markov class, the sequence of covers is defined
such that the random walk starts from the finest refinement and
proceeds to the coarsest one.  More specifically, consider a sequence
of covers such that each cover is a partition. Let $C_k$ be a
partition of $\CX_k^\infty$ and let $f_k : \CX^k \leftrightarrow C_k$
such that for each $x^k \in \CX^k$, there exists $f_k(x^k) \in
C_k$. Let $a \suffix b$ denote the fact that $a$ is a suffix of $b$
and let $F(\bx) \defn \cset{\bx' \in \CX^*}{\bx \suffix \bx'}$ be the
set of sequences for which $\bx$ is a suffix. Then $C_k =
\cset{F(\bx)}{\bx \in \CX^k}$. This could be an $n$-ary partition
tree, or more specifically, a suffix tree, if $|\CX| = n$. In that
case, there would be only stopping probability parameters $w$ and no
transition parameters $\tran$, since in a suffix tree, each node has
at most one child that contains $x^t$ for any time $t$.  The local
models $\prior$ can be defined via Dirichlet
priors~\citep[Sec.~9.8]{Degroot:OptimalStatisticalDecisions} on
$\CY$. In the binary case, this corresponds to
Example~\ref{ex:binary}.  In particular, the defined variable order
Markov model is identical to the formulation given
in~\citep{dimitrakakis:aistats:2010} and a generalisation
of~\citep{willems:context}.

\subsection{Conditional density estimation}
\label{sec:cde}

In conditional density estimation, a simple way to generate the
sequence of covers is to use a kd-tree to create sequence of
partitions of $\CX$. However, other methods, such as a cover tree are
easily applicable. As in the variable Markov model case, the random
walk starts from the finest cover (which corresponds to the deepest
part of the tree) and is subsequently coarsened. One particularly
interesting use of the flexibility offered by transition probabilities
here is to define {\em multiple density estimators} at each context.

For the density estimators in each context, we specifically consider
two alternatives. Firstly, a Normal-Wishart conjugate
prior~\cite[Sec.~9.10]{Degroot:OptimalStatisticalDecisions}. This is a
classical Bayesian estimator, which can be updated in closed
form. Secondly, a Bayesian tree density estimator that
straightforwardly extends~\cite{hutter:bayestree} from densities on
the $[0,1]$ interval to densities on $[0,1]^n$ through a
kd-tree. These alternatives are selected via the random
walk. Consequently, inference is performed on a double pseudo-tree.

\section{Related work}
\label{sec:related-work}
Among other things, the presented model relies upon a marginalisation
over a finite number of contexts for tractable inference. Similar
mechanisms have of course appeared before. It is nevertheless
worthwhile to note two recent models proposed
in~\citep{wong2010optional,hutter:bayestree}, which are directly
applied to density estimation on $\CX$. There, the selection of a
context $\ctx$ can be seen as a walk starting from the root node of a
tree, which corresponds to the whole of $\CX$ and proceeding to a
matching child node, which is one of the subsets of the root note,
stopping with some probability. These models are not trivially
applicable to {\em conditional} density estimation, apart from the
(perhaps naive) approach of estimating $p(x,y), p(x)$ separately and
using their ratio. On the other hand, they can naturally be
incorporated within our framework by using them as optional sub-models
performing density estimation in each context.

In the context of variable order Markov model estimation, a related
construction was presented in~\citep{dimitrakakis:aistats:2010}.
There, the process can be seen as a walk starting from the leaf node
of a suffix tree, stopping with some probability, otherwise proceeding
to the parent node. The same structure is implicitly present in the
classic context treee weighting method~\citep{willems:context}. The
proposed framework can be seen as an extension of those two methods
where the context structure is not limited to a partition tree.

Most of the work on conditional density estimation has focused on
kernel based methods and tree methods. For example, recently an
approximate kernel conditional density
estimation~\citep{isbell:mc:2008} has been developed which employs a
double tree structure for efficient estimation of the kernel
bandwidth.  Finally, a set of tree models for conditional density
estimation are surveyed in \citep{davies-interpolating}. However, none
of these methods is fully Bayesian, in the sense that a distribution
on models is not maintained. Rather, a single tree model is selected
after all the data has been seen. In that sense, the approach
suggested in this paper has the additional advantage of being
incrementally updatable in closed form.

Finally, it is worth mentioning the related problem of estimating
conditional probabilities in a large (but finite) sets. For this
problem, \cite{beygelzimer:cpte:uai} propose and analyse an efficient,
incremental tree-based method. 

\section{Numerical experiments}
\label{sec:experiments}

We examined the algorithm on a number of conditional density
estimation domains.  As previously mentioned in Sec.~\ref{sec:cde}, we
used a double pseudo-tree structure, with optional Normal-Wishart
conjugate priors for modelling densities. The prior weights were set
to $2^{-k}$ for contexts at depth $k$ in order to favour short trees,
while all transition probabilities were initially uniform.  Since
inference is closed form, we can update all parameters according to
Theorem~\ref{the:inference}.  In order to generate the covers
efficiently, we construct a set of kd-tree structures online. That is,
once more than $\theta_k$ observation are within a leaf node at depth
$k$, the node is partitioned along its largest dimension.  It is easy
to see that the (pseudo) tree depth, and consequently the complexity
of the method depends on the choice of $\theta_k$.
\begin{lemma}
  For a total of $T$ observations and $\theta_k \defn \alpha^k$,
  $\alpha > 1$ , the tree depth is bounded by $\BigO{\log_\alpha
    T(\alpha - 1)}$ and $\BigW{\log_{\alpha\beta} T(\alpha\beta -
    1)}$, where $\beta$ is a branching factor.
\end{lemma}
\begin{proof}
  Let us first consider the upper bound. The depth is maximal when
  the deepest leaf node is reached for every observation. Consequently,
  \begin{align*}
    T &=
    \sum_{k=0}^{d} \alpha^k 
    = \frac{\alpha^{d+1} - 1}{\alpha - 1},
  \end{align*}
  and so $d = \log_\alpha [1 + (\alpha - 1) T] -1$.
  We can obtain a lower bound by examining the case where the tree is
  balanced.  Then the number of nodes at depth $k$ is then $N_k = \bra^k$
  and consequently:
  \begin{align*}
    T
    &= \sum_{k=0}^d N_k \theta_k
    = \sum_{k=0}^d (\alpha \beta)^k,
  \end{align*}
  and so $d = \log_{\alpha\bra} [1 + (\alpha\bra - 1) T] -1$.
\end{proof}
Using this lemma, it is easy to see that the total complexity is
$\BigO{T \log T}$, thus only slightly worse than linear.

\subsection{An illustration}
\begin{figure*}[ht]
  \centering
\iftrue
  \subfigure[$10^3$ observations]{
    \includegraphics[width=0.45\textwidth]{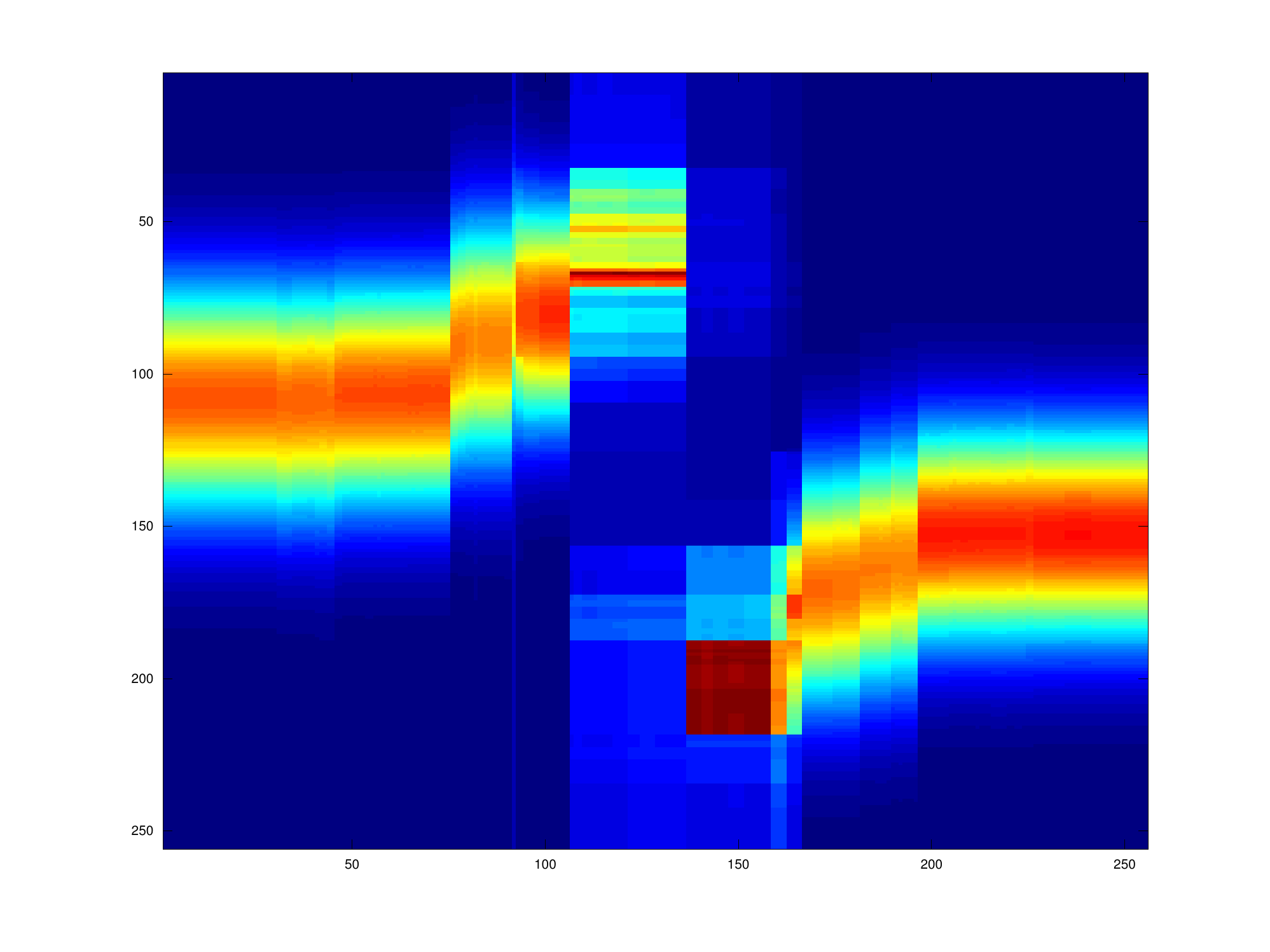}
  }
  \subfigure[$10^4$ observations]{
    \includegraphics[width=0.45\textwidth]{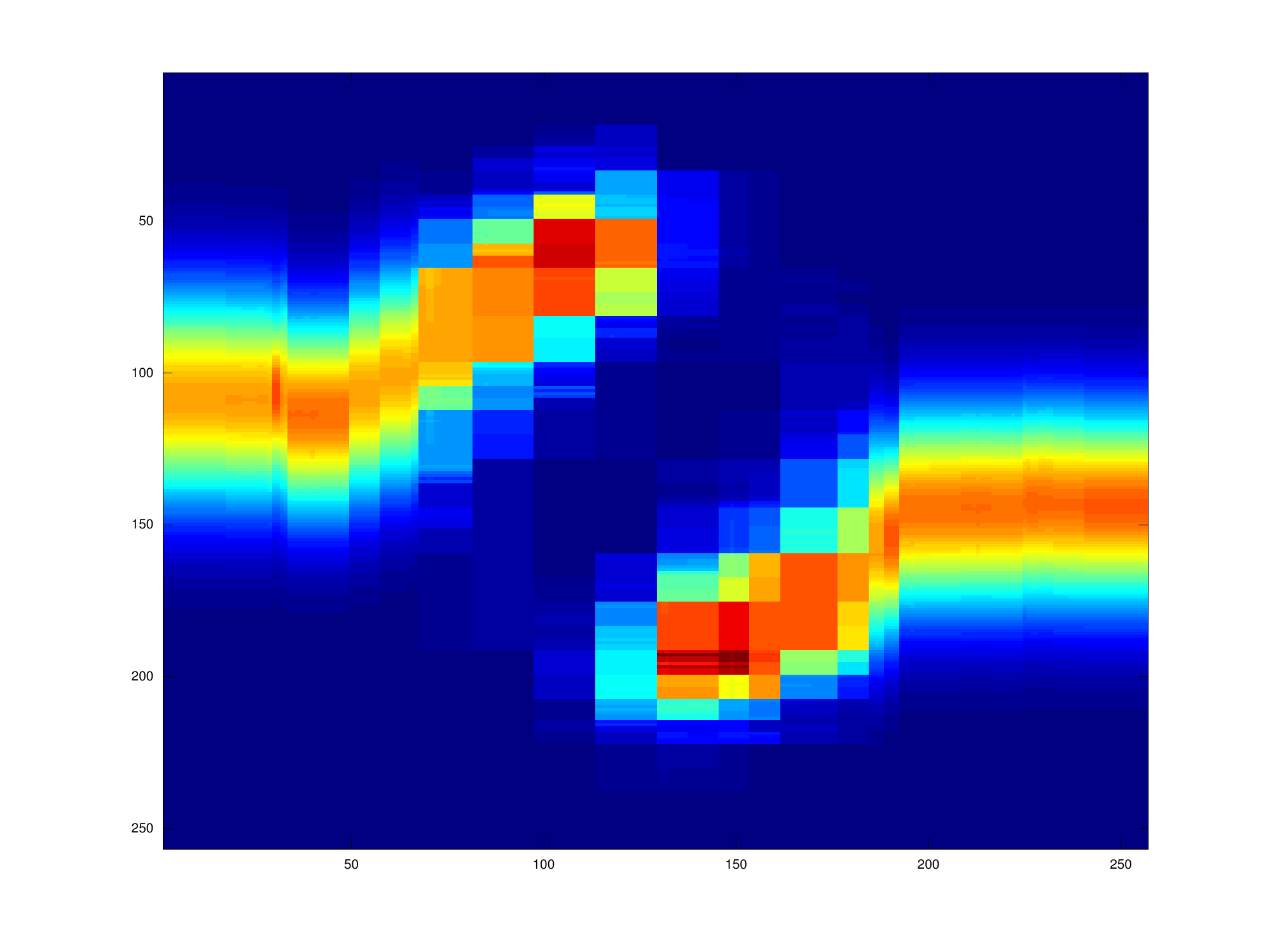}
  }
  \subfigure[$10^5$ observations]{
    \includegraphics[width=0.45\textwidth]{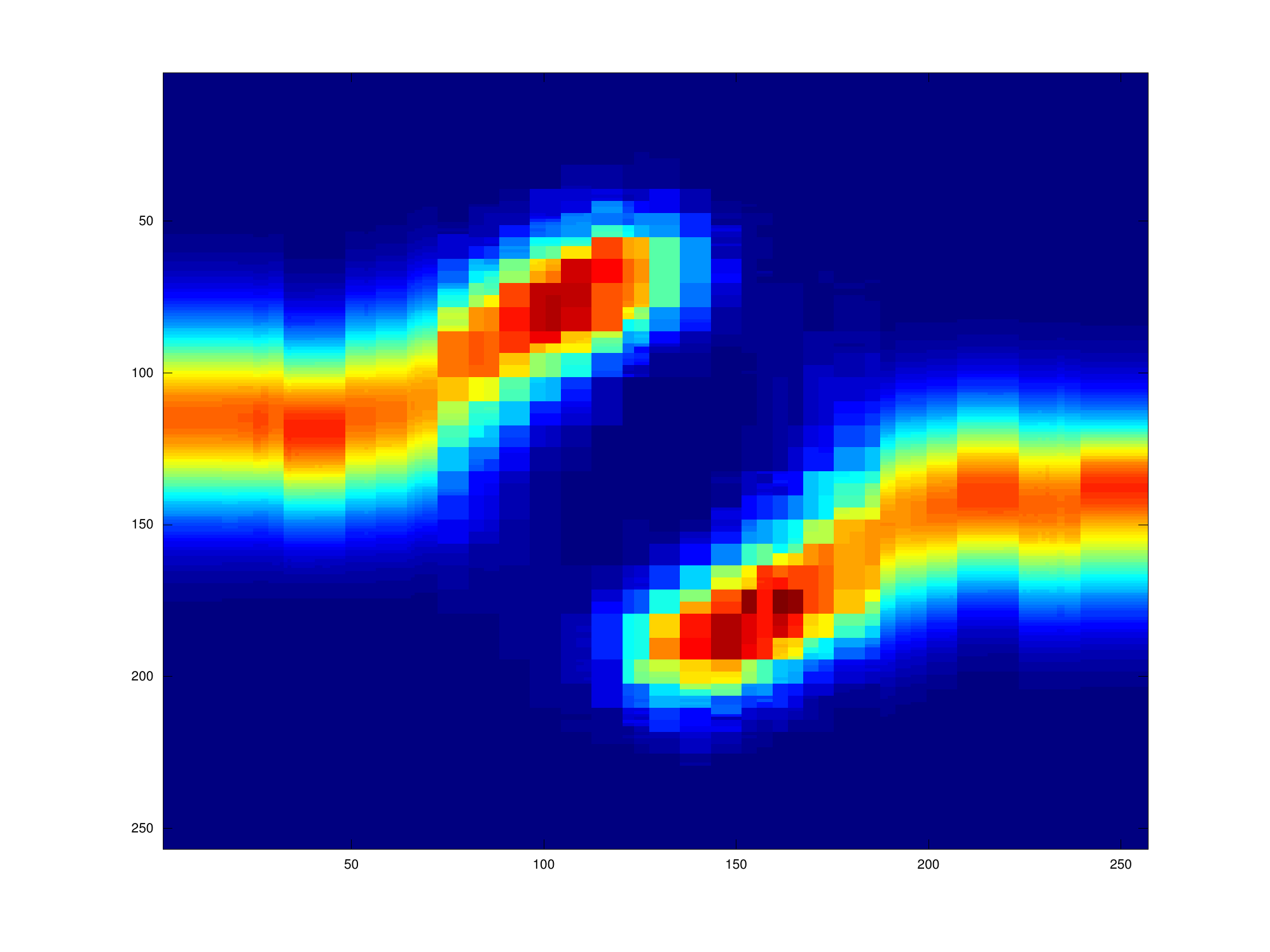}
  }
  \subfigure[$10^6$ observations]{
    \includegraphics[width=0.45\textwidth]{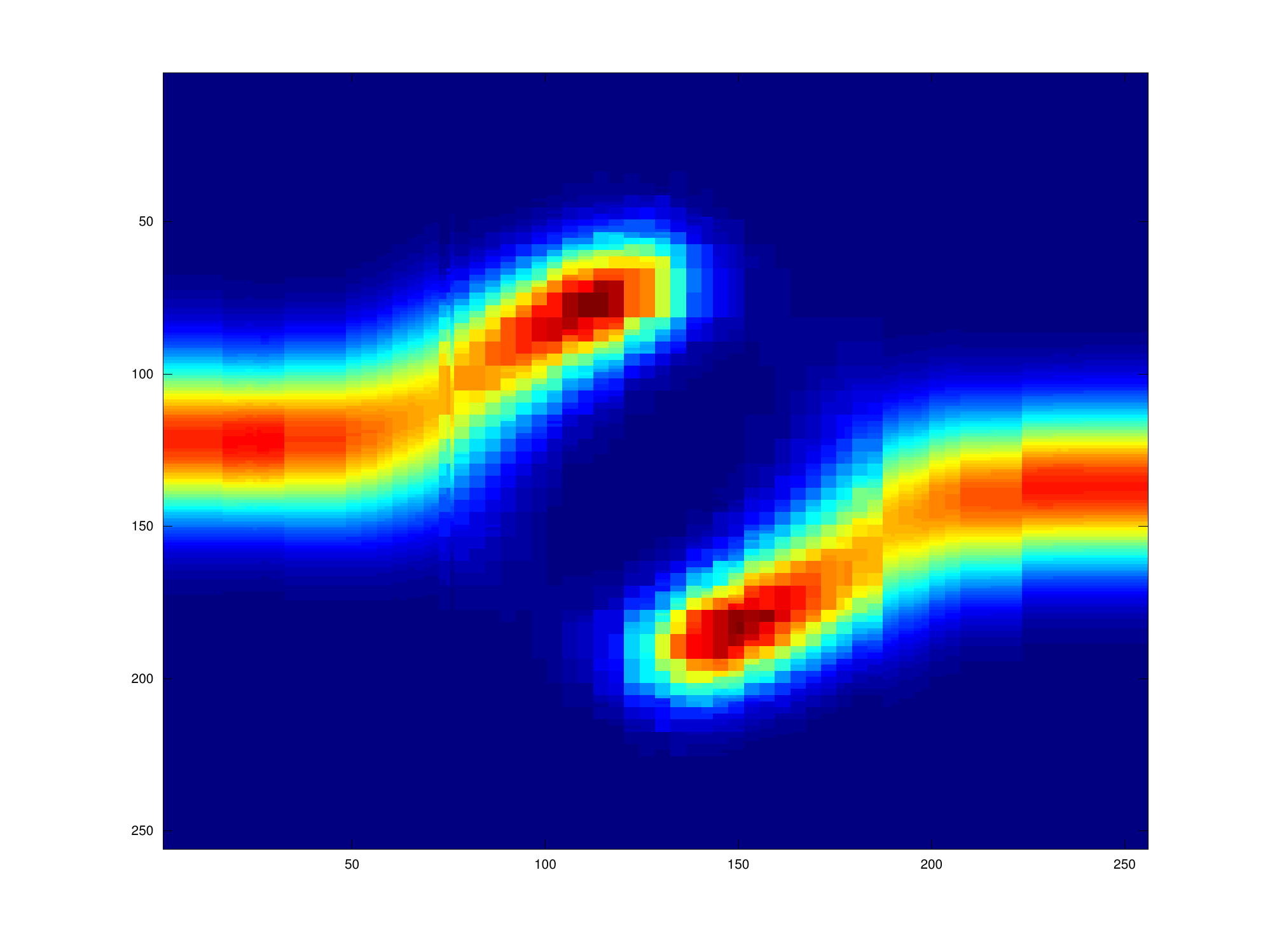}
  }
\fi
  \caption{Conditional density estimation illustration on a Gaussian
    ring distribution. It can be seen that the estimator settles on a
    Gaussian density near the edges, where the distribution is
    approximately normal, while uses a pseudo-tree distribution near
    the ring. The structure is refined with subsequent observations.}
  \label{fig:cde-illustration}
\end{figure*}

Figure~\ref{fig:cde-illustration} demonstrates the context model
estimator on a ring Gaussian distribution from which samples were
generated as follows. Firstly, the mean of a Gaussian was drawn by
sampling an angle $\theta$ from a mixture of univariate Gaussians. The
observation was then drawn from a bivariate Gaussian with mean equal
to the location on a unit ring determined by the drawn angle.
Consequently, near and within the ring, the distribution is highly
non-Gaussian, while further away from the ring the distribution
approaches normality. This is borne out in the figure, since, while in
far-away regions, the distribution is modelled with a smooth Gaussian,
close to the ring, even for a limited number of samples, the parts of
the model which correspond to non-Gaussian distributions have a higher
probability. 

\subsection{Comparisons}

We compared our method with a double-kernel conditional density
estimator utilising cross-validation for bandwidth selection.  This is
effectively a Parzen window estimator combined with a kernel density
estimator.  Although such methods are generally robust, they suffer
from two drawbacks. The first is the computational complexity
especially in terms of the bandwidth selection for the two
kernels. This is something addressed by~\cite{isbell:mc:2008}, which
uses a double tree structure to accelerate the search. The second and
most important drawback is that the bandwidth estimator is
invariant. This may potentially create problems, since ideally one
would like to vary the kernel in different parts of the space.  
For our quantitative experiments, we utilised a Gaussian kernel
throughout for the kernel estimators.

\begin{table}
  \centering
  \begin{tabular}{c|c|c|c|c}
    Name & $\CX$ & $\CY$ & training & holdout
    \\\hline
    Gaussian mixture & $\Reals$ & $\Reals$ & $10^6$ & $10^6$
    \\
    Uniform mixture & $\Reals$ & $\Reals$ & $10^6$ & $10^6$
    \\
    Geyser & $\Reals$ & $\Reals$ & $200$ & $72$
    \\
    Robot & $\Reals^{16}$ & $\Reals^8$ & $2812$ & $2644$
  \end{tabular}
  \caption{Summary of datasets}
  \label{tab:datasets}
\end{table}
\begin{figure*}
  \centering
  \subfigure[Gaussian mixture]{
    \psfrag{log loss}[B][B][0.7][0]{$\loss$}
    \psfrag{observations}[B][B][0.7][0]{$t$}
    \includegraphics[width=0.47\textwidth]{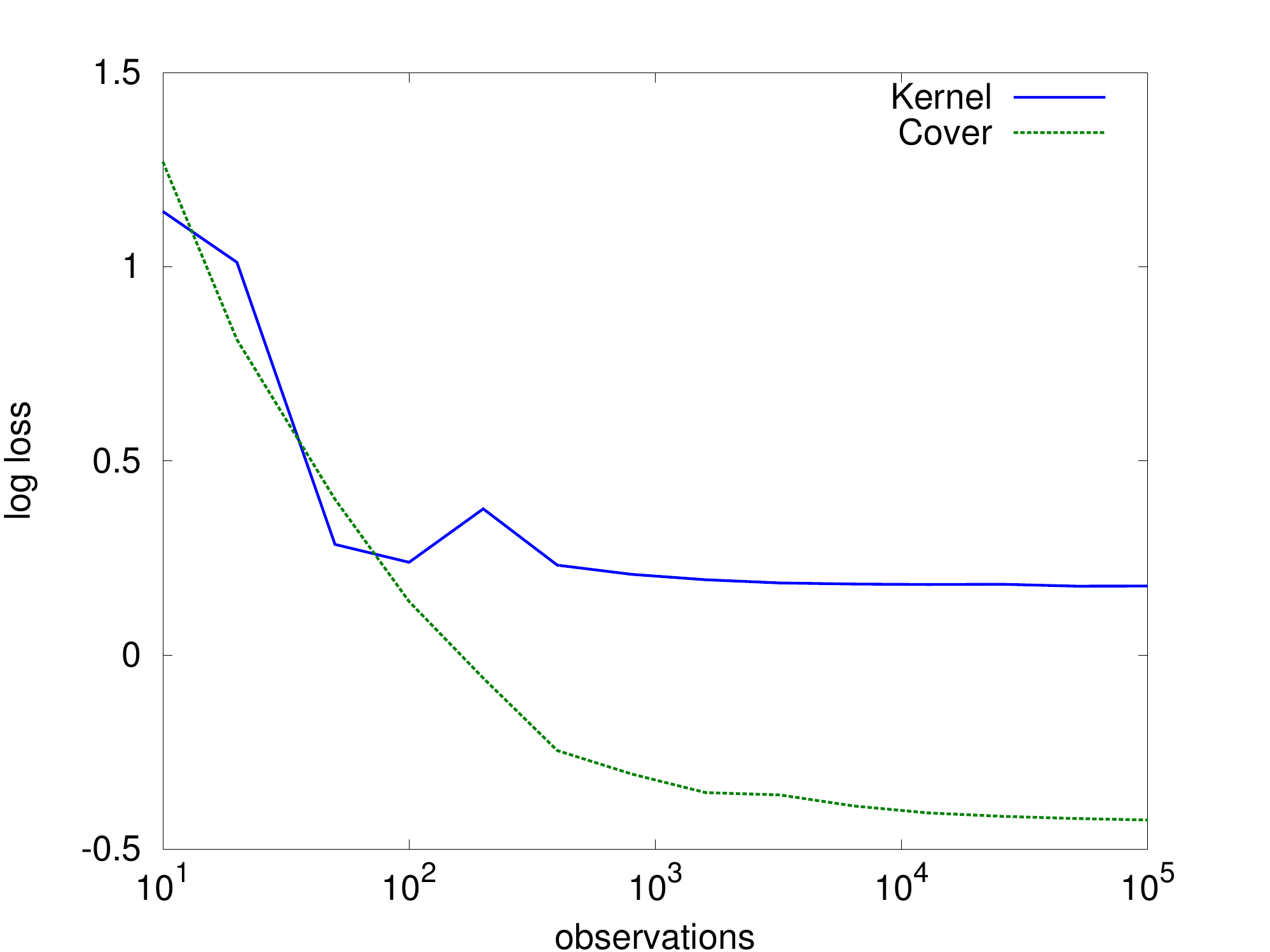}
  }
  \subfigure[Uniform mixture]{
    \psfrag{log loss}[B][B][0.7][0]{$\loss$}
    \psfrag{observations}[B][B][0.7][0]{$t$}
    \includegraphics[width=0.47\textwidth]{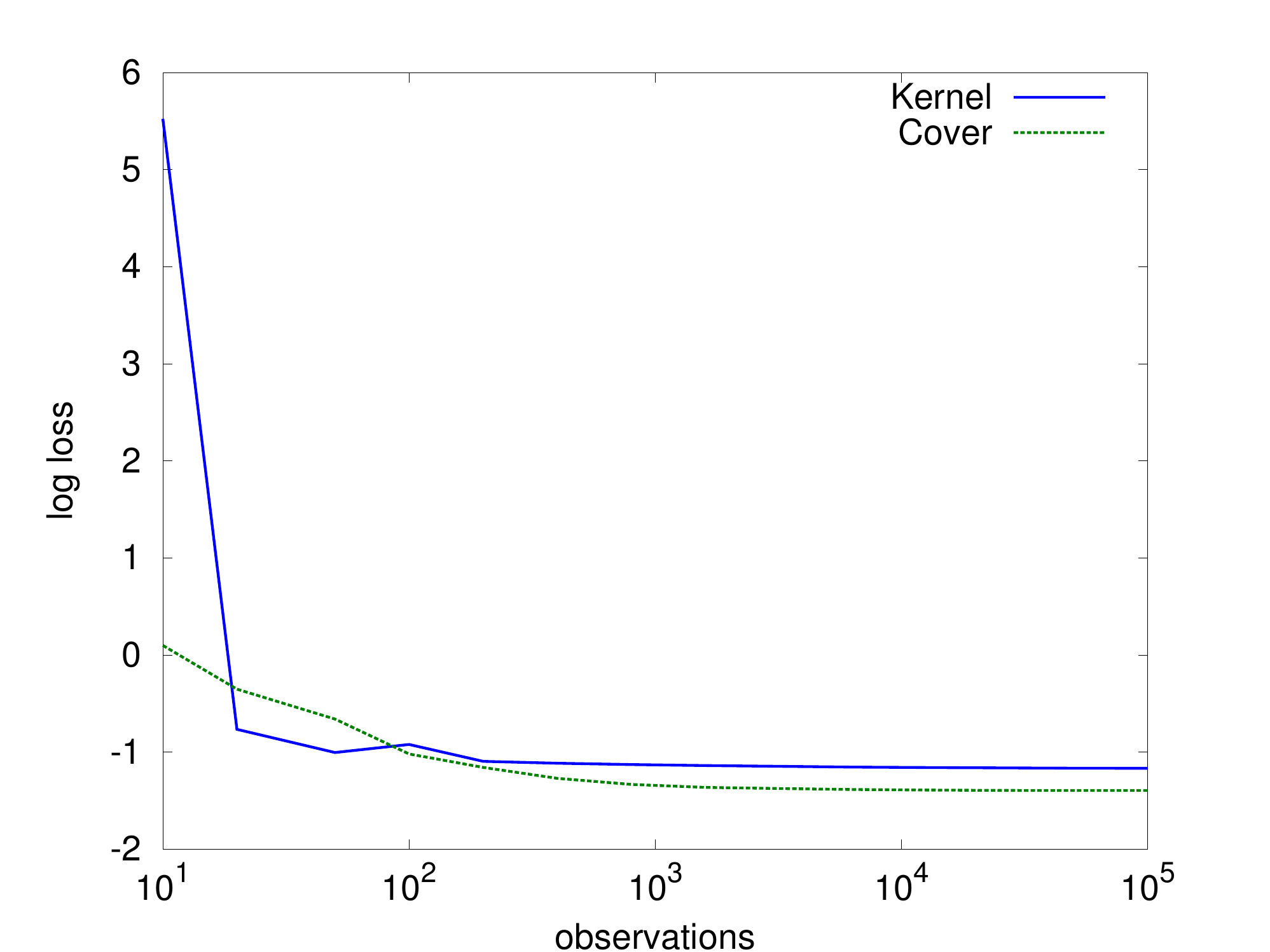}
  }
  \subfigure[Geyser]{
    \psfrag{log loss}[B][B][0.7][0]{$\loss$}
    \psfrag{observations}[B][B][0.7][0]{$t$}
    \includegraphics[width=0.47\textwidth]{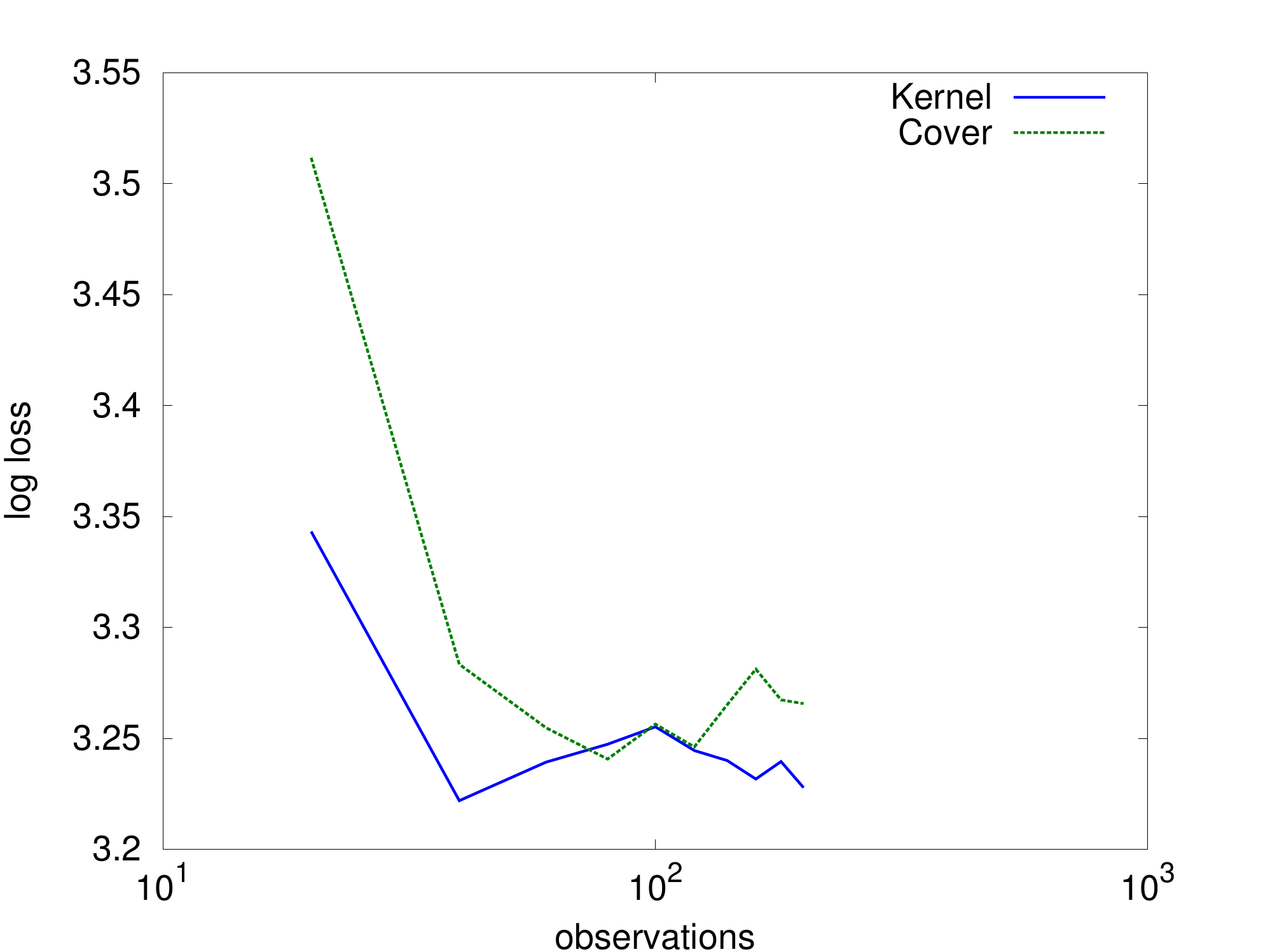}
  }
  \subfigure[Robot]{
    \psfrag{log loss}[B][B][0.7][0]{$\loss$}
    \psfrag{observations}[B][B][0.7][0]{$t$}
    \includegraphics[width=0.47\textwidth]{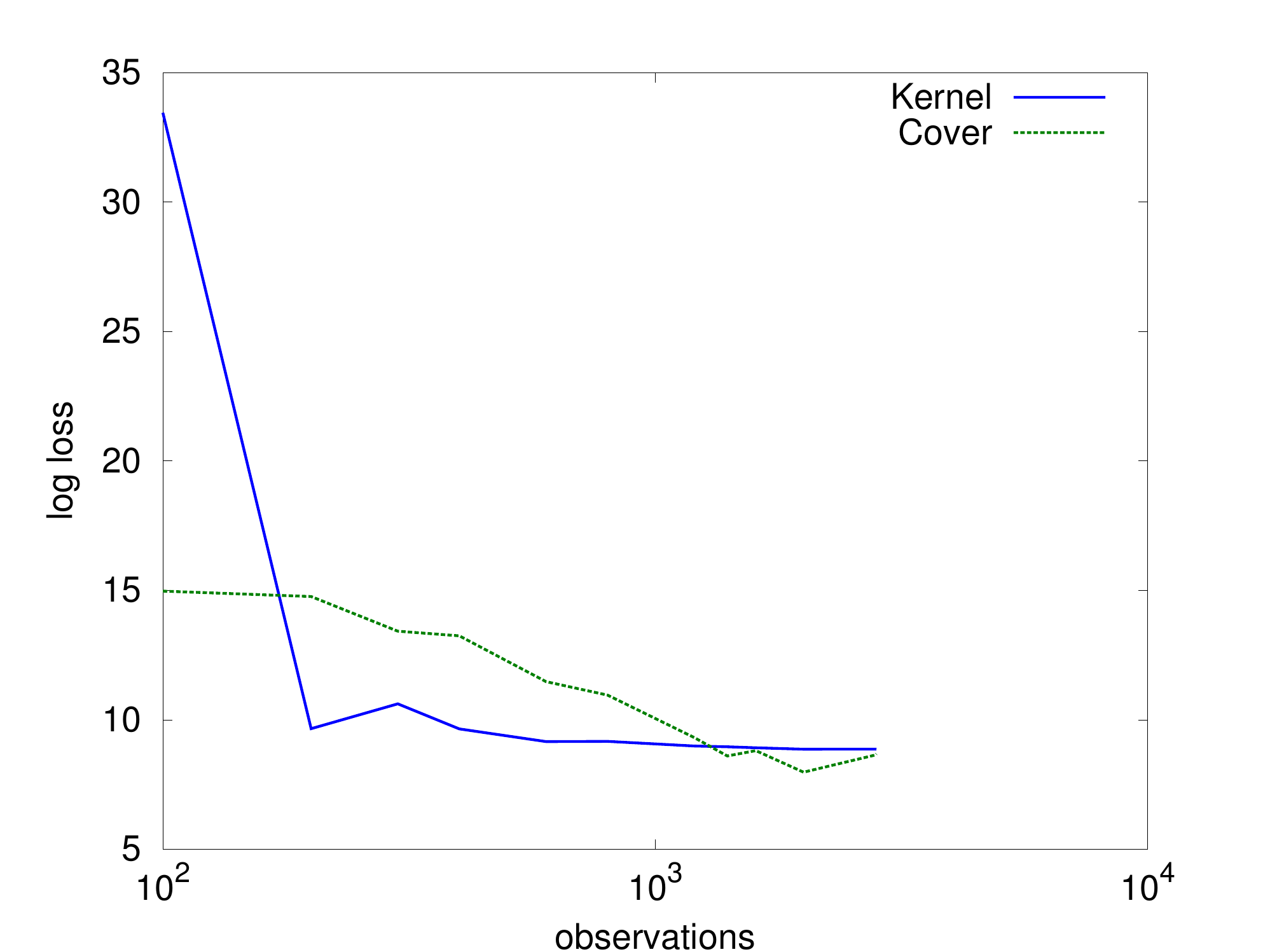}
  }
  \caption{Conditional density estimation performance on a hold-out
    set, for four different datasets as the number of observations $t$
    increases. The performance is in terms of the relative log loss
    $L_t$ or average negative log-likelihood of the hold-out set. In
    most cases, the context cover double pseudo-tree significantly
    outperforms a bandwidth-tuned kernel estimator.}
  \label{fig:cde-performance}
\end{figure*}

The experimens were performed on a number of datasets, summarised in
Table~\ref{tab:datasets}. The first two are large, synthetic
datasets. The {\em Gaussian mixture} dataset is a mixture of three
Gaussian distributions on $\Reals^2$, where the first dimension is
used as the conditioning variable.  Similarly the {\em Uniform
  mixture} dataset is a mixture of three uniform distributions.  We
also have results from two real datasets. The first, {\em Geyser}, is
the well-known dataset of eruption times and durations for the ``old
faithful'' geyser.  The second dataset, {\em Robot} is a set of
proximity sensor readings from a robot performing a navigation task.

For each dataset, we measured the average negative log-likelihood of
each method as the amount of training data increased.  Each dataset
$D$ was split into a training set $D_T$ and hold-out set $D_H$. For
each method, we obtained a sequence of conditional density models
$p_t$, trained on the subset $D_t \subset D_T$ of the first $t$
observations in the training set and then calculated the average
negative log-likelihood of that model on a hold-out set $D_h$:
\begin{equation}
  L_t \defn - \frac{1}{|D_h|} \sum_{(x,y) \in D_h} \ln p_t(y \mid x).
\end{equation}

For the cover method, we employed the same settings as in the previous
experiment. For the double-kernel method, for each training subset
$D_t$, we employed 10-fold cross-validation to select the bandwidths
of the two kernels and then used the chosen bandwidths to obtain a
model on the full subset $D_t$. The criterion for choosing the
bandwidth was the likelihood on the left-out folds.

Figure~\ref{fig:cde-performance} compares the performance of our model
with a double-kernel conditional density. One would expect the kernel
method to perform best in the {\em Gaussian mixture} dataset, while
the cover method would be favoured in the uniform mixture.  This
however, is clearly not the case. Firstly, note that the cover method
can optionally use a Normal-Wishart distribution to model the density
at any part of the space. Thus, the pure Gaussian kernel has no
initial advantage.  Secondly, some parts of $\CX$ have much fewer
samples and so would require a much wider kernel for accurate
estimation. However, the use of an invariant kernel means that this is
not possible. In the {\em uniform mixture} dataset, the kernel method
is almost as well as the cover method, though it is initially
disadvantage due to the bad fit of the Gaussian kernel to the uniform
blocks. In the widely-used, although extremely small, {\em Geyser}
dataset, it can be seen that the kernel method dominates the cover
one.  However, the difference is quite small and the size of the
dataset is such that the performance of the method is mainly dependent
upon how well its prior assumptions match the dataset. Finally, in the
{\em Robot} dataset, which is high-dimensional but has only a moderate
number of observations, the methods are more or less evenly
matched. The initially bad performance of the kernel method is mainly
due to the fact that it is hard to choose a good bandwith from only
100 samples in a high-dimensional space.  

Overall, one may observe that the two methods usually perform mostly
similarly.  However, the cover method appears to be more robust and in
some cases its asymptotic performance is significantly better than
that of the kernel method.

\section{Conclusion}
\label{sec:conclusion}

We outlined an efficient, online, closed-form inference procedure for
estimation on a sequence of covers.  It can be seen as a direct
extension of a previous
construction~\citep{dimitrakakis:aistats:2010}, which was limited to
partition trees and an analogous procedure for density estimation on
partition trees, given by~\citet{hutter:bayestree}. 

In principle, the approach is applicable to any problem involving
estimation of conditional measures, such as classification and
variable order Markov model estimation. As an example, we applied it
to conditional density estimation, a fundamental problem in
statistics. The result is the first, to our knowledge, closed-form,
incremental, polynomial-time, Bayesian conditional density estimation
method.

In order to do this, we utilised a double pseudo-tree structure.  The
first part of the structure was used to estimate the conditional
probabilities of context models. The second part of the structure was
used to estimate a density for each context.  This resulted in a
procedure for closed-form, Bayesian, non parametric conditional
density estimation. As expected, the performance of this method was in
some cases significantly better than that of a kernel based estimator
with an invariant kernel.

In future work, we would like to consider other density estimators for
the local context models. Since there are virtually no restrictions
regarding their type (other than the ability for incremental
conditioning), using kernel density estimators on each context
instead, could be a route towards obtaining non-invariant kernel
density estimation methods.  In addition, it would be interesting to
consider problems where we have some prior information regarding the
smoothness of the underlying conditioning density, perhaps in terms of
Lipschitz conditions with respect to the conditioning variable.

The main open problem is how to generate the covers.  In this paper,
we utilised a kd-tree to do so. However, the generality of the
approach is such that many other more interesting alternatives are
possible.  For example, cover trees~\citep{cover-tree:icml2006}, which
are an extremely efficient nearest-neighbour method, are an ideal
alternative. This alternate structure, would allow the application of
cover models to an arbitrary metric space. In addition, inference on
any lattice structure should remain tractable.

Nevertheless, the problem of finding a suitable sequence of covers
remains. This is more pronounced for controlled processes, because one
cannot rely on the statistics of the observations to create a useful
cover. This problem can be circumvented if a distribution on covers is
maintained, which would be more in the spirit of the optional
P{\'o}lya tree~\citep{wong2010optional}. However, then inference would
no longer be closed form.

\section*{Acknowledgments}
Many thanks to Peter Auer for pointing out that the original variable
order Markov model construction is generalizable, and to Peter
Gr\"{u}nwald, Marcus Hutter and Ronald Ortner for extremely useful
discussions. Finally, thanks go to the anonymous reviewers who
provided thoughtful comments for previous versions of this paper.

\end{document}